\documentclass{article}

\usepackage{microtype}
\usepackage{graphicx}
\usepackage{subfigure}
\usepackage{booktabs} %

\usepackage{hyperref}

\usepackage[accepted]{icml2023}

\usepackage{amsmath}
\usepackage{amssymb}
\usepackage{mathtools}
\usepackage{amsthm}
\usepackage{bbm}

\usepackage[capitalize,noabbrev]{cleveref}

\theoremstyle{plain}
\newtheorem{theorem}{Theorem}[section]
\newtheorem{proposition}[theorem]{Proposition}
\newtheorem{lemma}[theorem]{Lemma}

\theoremstyle{definition}

\theoremstyle{remark}

\usepackage[textsize=tiny]{todonotes}

\usepackage{amsmath,amsfonts,bm}

\def\eqref#1{equation~\ref{#1}}

\def\floor#1{\lfloor #1 \rfloor}
\def\1{\bm{1}}

\DeclareMathAlphabet{\mathsfit}{\encodingdefault}{\sfdefault}{m}{sl}
\SetMathAlphabet{\mathsfit}{bold}{\encodingdefault}{\sfdefault}{bx}{n}

\def\gA{{\mathcal{A}}}

\def\gE{{\mathcal{E}}}
\def\gF{{\mathcal{F}}}
\def\gG{{\mathcal{G}}}

\def\gL{{\mathcal{L}}}

\def\gN{{\mathcal{N}}}

\def\gP{{\mathcal{P}}}

\def\gR{{\mathcal{R}}}

\def\gX{{\mathcal{X}}}
\def\gY{{\mathcal{Y}}}
\def\gZ{{\mathcal{Z}}}

\def\sR{{\mathbb{R}}}

\newcommand{\E}{\mathbb{E}}

\DeclareMathOperator*{\argmax}{arg\,max}

\renewcommand{\epsilon}{\varepsilon}

\DeclareMathOperator*{\diag}{diag}

\def\ie{\textit{i.e.,~}}
\def\eg{\textit{e.g.,~}}

\begin{document}

\twocolumn[
\icmltitle{On the Generalization of Multi-modal Contrastive Learning}

\icmlsetsymbol{equal}{*}

\begin{icmlauthorlist}
\icmlauthor{Qi Zhang}{equal,1}
\icmlauthor{Yifei Wang}{equal,2}
\icmlauthor{Yisen Wang}{1,3}
\end{icmlauthorlist}

\icmlaffiliation{1}{National Key Lab of General Artificial Intelligence, School of Intelligence Science and Technology, Peking University}
\icmlaffiliation{2}{School of Mathematical Sciences, Peking University}
\icmlaffiliation{3}{Institute for Artificial Intelligence, Peking University}

\icmlcorrespondingauthor{Yisen Wang}{yisen.wang@pku.edu.cn}

\icmlkeywords{Machine Learning, ICML}

\vskip 0.3in
]

\printAffiliationsAndNotice{\icmlEqualContribution}

\begin{abstract}
Multi-modal contrastive learning (MMCL) has recently garnered considerable interest due to its superior performance in visual tasks, achieved by embedding multi-modal data, such as visual-language pairs. However, there still lack theoretical understandings of how MMCL extracts useful visual representation from multi-modal pairs, and particularly, how MMCL outperforms previous approaches like self-supervised contrastive learning (SSCL). In this paper, by drawing an intrinsic connection between MMCL and asymmetric matrix factorization, we establish the first generalization guarantees of MMCL for visual downstream tasks. Based on this framework, we further unify MMCL and SSCL by showing that MMCL implicitly performs SSCL with (pseudo) positive pairs induced by text pairs. Through this unified perspective, we characterize the advantage of MMCL by showing that text pairs induce more semantically consistent and diverse positive pairs, which, according to our analysis, provably benefit downstream generalization.
Inspired by this finding, we propose CLIP-guided resampling methods to significantly improve the downstream performance of SSCL on ImageNet by leveraging multi-modal information. Code is available at \url{https://github.com/PKU-ML/CLIP-Help-SimCLR}.
\end{abstract}

\section{Introduction}
Recently, multi-modal contrastive learning (MMCL), including CLIP \citep{clip} and its variants \citep{declip,mu2022slip,yao2021filip}, has achieved impressive performance for visual representation learning, and transfer well to various downstream tasks like zero-shot and few-shot image classification. The core idea of MMCL is rather simple, which aligns the samples of the same image-text pairs together while pushing away other unrelated samples in the latent feature space. 
However, it remains not fully clear to us why matching multi-modal pairs would benefit visual representation learning, and what are the key factors that affect its downstream performance.

Meanwhile, another popular scenario for contrastive learning is self-supervised learning, which also obtains competitive performance recently \cite{simclr,moco,wang2021residual}. Nevertheless, recent MMCL methods (like CLIP) have shown significant advantages over its self-supervised contrastive learning (SSCL) counterparts like SimCLR \cite{simclr}. Existing theories of SSCL \citep{arora,haochen,wang2020understanding} only establish the optimality of self-supervised representations on downstream tasks, and fail to characterize why MMCL could outperform SSCL. Another major obstacle is the generation process of data pairs. In particular, positive pairs in SSCL are visual-only samples generated by random data augmentations of the raw image; Instead, the positive pairs in MMCL are multi-modal (\eg visual-language) pairs directly provided by the dataset. Since existing SSCL theories rely crucially on the assumption that data augmentations produce overlap between visual samples \cite{wang2022chaos,saunshi2022understanding}, they cannot be directly applied to MMCL that relies on multi-modal data pairs.

In this paper, we propose the first theoretical analysis on the generalization ability of MMCL. To achieve this, we establish an equivalence between the MMCL objective and the asymmetric matrix factorization (AMF) of the multi-modal co-occurrence matrix. Built upon this connection, we characterize the ideal pretrained representations of MMCL and its generalization bounds on visual and language downstream tasks,
where the bounds are influenced by the properties of the multi-modal co-occurrence matrix, for example, its singular value.

The established theoretical framework also allows us to characterize the difference between MMCL and SSCL under a unified perspective. To be specific, we first formally unify MMCL and SSL under the framework of uni-modal similarity graphs, where language pairs in MMCL can be regarded as a special kind of data augmentation for generating positive visual pairs. Based on this perspective, we compare MMCL and SSCL on real-world data and show that text-induced positive pairs have better semantic consistency and diversity than augmentation-based ones in SSCL, which explains the superiority of MMCL on downstream tasks. Besides the empirical comparisons, we theoretically analyze this difference by modeling the data generation process with the hierarchical random graph \cite{clauset2008hierarchical}. Based on this understanding, we further leverage multi-modal information in CLIP to guide the self-supervised visual learning with SimCLR on ImageNet and achieve significant improvements, which validates our understanding of the superiority of multi-modal positive pairs.

We summarize our contributions as follows:
\begin{itemize}
    \item We establish the first generalization theoretical guarantee for multi-modal contrastive learning (MMCL). We provide a new perspective of the multi-modal contrastive loss by connecting it with an asymmetric matrix decomposition objective.  
    \item We provide a unified perspective for understanding the connections and differences between multi-modal and self-supervised contrastive learning. Based on this perspective, we examine their differences on real-world data, and find that multi-modal information induces better positive visual pairs than self-supervision (with better semantic consistency and diversity), which explains the superiority of MMCL. 
    \item As a verification of our understanding above, we further investigate a new scenario where we leverage multi-modal information in pretrained models (like CLIP) to guide the sampling process in self-supervised learning like SimCLR. We propose four different techniques and they both bring improvements (as much as 6.2\%) on ImageNet.
\end{itemize}

\section{Related Work}
\textbf{Multi-modal Pretraining Applications.}
Traditional single-stream models \citep{lu2019vilbert,li2019visualbert} have been widely discussed and shown the impressive performance in various multi-modal tasks. However, as they do not have independent encoders for different modals, the transferability of these frameworks is usually limited. On contrast, multi-modal contrastive learning paradigms represented by CLIP \citep{clip} have recently obtained the promising performance in multi-modal downstream tasks including zero-shot learning, finetuning and linear-probing. Inspired by CLIP, various variants are proposed to improve the efficiency and performance of multi-modal pretraining. SLIP \citep{mu2022slip} and DeCLIP \citep{declip} combine the self-supervised and multi-modal contrastive learning to accelerate the training process. FILIP \citep{yao2021filip} propose fine-grained multi-modal contrastive objective to make the encoder focus more on the local features.

\textbf{Theory of Contrasative Learning.} Motivated by the empirical success of the contrastive objective, many researchers try to theoretically analyze how it works. \citet{wang2020understanding} understand the contrastive loss from two terms in it: the alignment of the positive samples and the uniformity of the negative samples. \citet{infomax} analyze the objective from the mutual information theory. \citet{arora} establish the theoretical guarantee between the pretraining contrastive loss and the downstream classification performance. 
\citet{haochen} revisit the contrastive objective from a spectral graph perspective, which explains the relationship between the augmented samples and the downstream performance of contrastive learning. \citet{wang2022chaos,wang2023message} provide a theoretical understanding for contrastive learning from the perspective of augmentation overlap and message passing respectively. As these prior theoretical works mainly focus on the single-modal contrastive learning, the theoretical analysis on the multi-modal contrastive learning is still quite limited. In this work, we theoretically analyze the relationship between the design of the multi-modal contrastive  paradigms and its generalization ability on downstream tasks.

\textbf{Theory of Multi-modal Learning.}
For the theoretical analysis of multi-modal learning, there are few related works. 
\citet{sun2020tcgm} propose a information-theoretic framework and prove that their method can learn ground-truth Bayesian posterior classifier for each modality and the
Bayesian posterior aggregator for all modalities. \citet{huang2021makes} proves that the multi-modal models can learn better representations than single-modal models in certain conditions. However, both of their analysis do not focus on the multi-modal \emph{contrastive} paradigm and can not explain why the contrastive methods can achieve such an impressive performance.

\section{Generalization Theory of Multi-Modal Contrastive Learning}
\label{sec:downstream guarantees}
\subsection{Mathematical Formulation}
We start by introducing the basic mathematical formulation for multi-modal contrastive learning. Without loss of generality, taking CLIP \cite{clip} for an example, we have the paired data $(x_v,x_l)$ from the visual domain ($x_v$ denotes an image) and the language domain ($x_l$ denotes a corresponding text description of the image). Each $x_v$ or $x_l$ belongs to one of $r$ classes. We use 
 $\mathcal{X}_V$ to denote the set of all visual data with distribution $\mathcal{P}_V$, and $\mathcal{X}_L$ to denote the set of all language data with distribution $\mathcal{P}_L$. Their joint multi-modal distribution is $\gP_M$. For ease of exposition, we assume $\gX_V,\gX_L$ to be finite but exponentially large sets\footnote{With some non-essential nuances as in \citet{haochen}, our analysis can also be extended to the infinite data setting.}, and denote
 $N_V=\vert \mathcal{X} _V \vert$ and $N_L=\vert \mathcal{X} _L \vert$. 
 The goal of multi-modal contrastive learning is to obtain a joint embedding of the visual data $\gX_V$ and language data $\gX_L$ in the $k$-dimensional latent space $\gZ\in\gR^k$ by learning a visual encoder $f_V:\gX_V\to\gZ$ and a language encoder $f_L:\gX_L\to\gZ$, such that semantically similar samples (either image-image, text-text or image-text pairs) have close representations, and different samples are apart. A recent work \cite{tschannen2022image} also explores a Siamese network, \ie $f_V=f_L$. Here we consider the general case with two different encoders. 
 
For multi-modal positive and negative pairs, we define an image-text pair drawn from the paired visual-language data, \ie $(x_v,x_l)\sim\gP_M$, as positive pairs, and draw independent samples from each domain, $x^-_v\sim\gP_V,x^-_l\sim\gP_L$, and treat $(x_v,x^-_l)$, $(x^-_v,x_l)$ and $(x^-_v,x^-_l)$ as negative pairs, because samples in these pairs are independent of each other.

Given positive and negative pairs $(x_v,x_l,x^-_v,x^-_l)$, one popular learning objective is the symmetric cross entropy (SCE) loss (adopted in CLIP) calculated over similarity scores:
\begin{equation}
\begin{aligned}
    \mathcal{L}_{\rm SCE}(f_V,f_L)
    =&-\mathbb{E}_{x_v,x_l}\log\frac{\exp\left(f_V(x_v)^\top f_L(x_l)\right)}{\E_{x_l^-}\exp(f_V(x_v)^\top f_L(x_l^-))}\\
    &-\mathbb{E}_{x_v,x_l}\log\frac{\exp\left(f_V(x_v)^\top f_L(x_l)\right)}{\E_{x_v^-}\exp(f_V(x_v^-)^\top f_L(x_l))}.    
\end{aligned}
\label{eqn:infonce loss}
\end{equation}
This objective can be seen as an extension of the popular InfoNCE loss \cite{InfoNCE} to the multi-modal scenario \cite{zhang2020contrastive}. During the learning process, positive pairs $(x_v,x_l)$ are pulled together in the latent space while negative pairs $(x_v,x^-_l)$ and $(x^-_v,x_l)$ are pushed apart. Following the same spirit, we consider a similar multi-modal spectral loss for the ease of theoretical analysis,
\begin{equation}
\begin{aligned}
&\mathcal{L}_{\rm SCL}(f_V,f_L) \\
= &-2\mathbb{E}_{x_v,x_l}f_V(x_v)^\top f_L(x_l)+
\mathbb{E}_{x^-_v,x^-_l}(f_V(x_v^-)^\top f_L(x_l^-))^2.
\end{aligned}
\label{eqn:spectral loss}
\end{equation}
Comparing Eq.~\ref{eqn:infonce loss} and Eq.~\ref{eqn:spectral loss}, we can easily see that the two objectives have the same loss for positive pairs, and only differ at the specific loss function used for pushing negative pairs apart ($\operatorname{logsumexp}$ loss in Eq.~\ref{eqn:infonce loss} v.s. $\ell_2$ loss in Eq.~\ref{eqn:spectral loss}). The multi-modal spectral loss can be regarded as an extension of the visual spectral contrastive loss that achieves comparable performance to the InfoNCE loss in visual tasks \cite{haochen}. Nevertheless, their analysis can only be applied to self-supervised contrastive learning where positive and negative pairs come from the same domain.

After pretraining, we evaluate the learned representations by applying them to downstream tasks. Taking the visual linear probing task as an example, 
we train a linear classifier to predict class labels $y\in\gY$ from the output features of $f_V$ by $g_{f,B_V}(x_v) = \argmax_{i\in [r]}(f_V(x_v)^\top B_V)_i$, where $B_V\in\sR^{k\times r}$ denotes the weight matrix. The linear probing error of $f_V$ is defined as the error of the optimal linear classifier on the encoded features, \ie
\begin{equation}
    \mathcal{E}(f_V) = \min_{B_V} \mathbb{E}_{x_v \sim \mathcal{P}_V} \mathbbm{1} [g_{f,B_V}(x_v) \neq y(x_v)],
\end{equation}
where $y(x_v)$ denotes the label of $x_v$. Likewise, we can define the linear probing error $\gE(f_L)$ for the text classification.

\subsection{An Asymmetric Matrix Factorization View of Multi-modal Contrastive Learning}

With its samplewise pretraining objective (Eqs.~\ref{eqn:infonce loss} \& \ref{eqn:spectral loss}), multi-modal contrastive learning (MMCL) is usually understood as an instance-level feature matching task between visual and language domains \cite{clip}. However, little is known about {the overall distribution} of the learned features, which hinders us from understanding how its instance-level pretraining benefits downstream applications. In this section, with a reformulation of the MMCL objective, we show that MMCL is essentially equivalent to the asymmetric matrix factorization (AMF) of the joint data distribution $\gP_M(x_v,x_l)$. AMF is an important class of methods in classical machine learning with inherent connections to PCA, K-means, and spectral clustering \cite{ding2005equivalence}, and is widely adopted in unsupervised learning scenarios like Latent Semantic Analysis \cite{deerwester1990indexing} and word embedding \cite{pennington2014glove}. Generally speaking, AMF can extract low-frequency components that underline the common structure of the joint distribution, which is helpful for MMCL analysis.

We start by formulating the joint distribution $\gP_M(x_v,x_l)$ as a \emph{co-occurrence matrix} $P_M\in\mathbb{R}^{N_V\times N_L}$ between all visual-language data pairs, where
\begin{equation}
 \left(P_M\right)_{x_v,x_l}=\gP_M(x_v,x_l)\geq0,\ \forall\ x_v\in[N_V],x_l\in[N_L].
\end{equation}
We can see that $P_M$ is a non-negative asymmetric matrix that can be exponentially large. A canonical assumption of representation learning is that high-dimensional data (like images and text) lie in a low-dimensional manifold. Then, we consider the following low-rank matrix factorization for the \textit{normalized co-occurrence matrix} $\tilde{P}_M$:
\begin{equation}
    \gL_{\rm AMF}(F_V,F_L)=\|\tilde{P}_M-F_VF_L^\top\|^2,
    \label{eqn:amf-objective}
\end{equation}
where $F_V\in\gR^{N_V\times k},F_L\in\gR^{N_L\times k}$ are factorized low-rank components $(k\ll\min(N_V,N_L))$ of the visual and language domains, respectively. To obtain the normalized co-occurrence matrix $\tilde{P}_M$, we adopt two-side normalization
\begin{equation}
 (\tilde{P}_M)_{x_{v},x_l}=\frac{\gP_M({x_v,x_l})}{\sqrt{\gP_V({x_v})\gP_L({x_l})}},   
 \label{eqn:normalization}
\end{equation}
where $\gP_V(x_v)=\sum_{x_l}\gP_M(x_v,x_l)$ denotes the marginal probability of $x_v$, and $\gP_L(x_l)=\sum_{x_v}\gP_M(x_v,x_l)$ denotes the marginal probability of $x_l$. 
Based on this formulation, we are ready to establish the key result of this paper.

\begin{theorem}[Equivalence]
Let the $x_v$-row of $F_V$ and the $x_l$-row of $F_L$ represent the corresponding encoded features of these samples in the following form,
\begin{subequations}
\begin{align}
(F_V)_{x_v} &= \sqrt{\mathcal{P}_V(x_v)}f_V(x_v)^\top,\\
(F_L)_{x_l} &= \sqrt{\mathcal{P}_L(x_l)}f_L(x_l)^\top.    
\end{align}
\label{eqn:U-V-formulation}
\end{subequations}
Then low-rank asymmetric matrix factorization loss (Eq.~\ref{eqn:amf-objective}) is equivalent to the multi-modal contrastive loss (Eq.~\ref{eqn:spectral loss}) up to a constant,
\begin{align}
\mathcal{L}_{\rm AMF}(F_V,F_L) = \mathcal{L}_{\rm SCL}(f_V,f_L) +const.
\end{align}
\label{thm:spectral=asymetric}
\end{theorem}
\begin{proof}
Taking the definition of $F_V$ and $F_L$ in Eq.~\ref{eqn:U-V-formulation} into the decomposition loss $\gL_{\rm AMF}(F_V,F_L)$, and combing with the definition of $\tilde{P}_M$ in Eq.~\ref{eqn:normalization}, we have
\begin{align*}
&\gL_{\rm AMF}(F_V,F_L) \\
=& \Vert \tilde{P}_M - F_VF_L^\top \Vert ^2\\
        =&\sum\limits_{x_v,x_l}\bigg( \frac{\gP_M(x_v,x_l)}{\sqrt{\gP_V(x_v)\gP_L(x_l)}}\\
        &~~~~~~~~~~~-\sqrt{\gP_V(x_v)}f_V(x_v)^\top \sqrt{\gP_L(x_l)}f_L(x_l) \bigg)^2\\
        =&\sum\limits_{x_v,x_l}\bigg(\frac{\gP_M(x_v,x_l)^2}{\gP_V(x_v)\gP_L(x_l)} -2\gP_M(x_v,x_l)f_V(x_v)^\top f_L(x_L)\\
        &~~~~~~~~~~~+\gP_V(x_v)\gP_L(x_l)\left(f_V(x_v)^\top f_L(x_L)\right)^2\bigg)\\
        =&\underbrace{\sum\limits_{x_v,x_l}\bigg(\frac{\gP_M(x_v,x_l)^2}{\gP_V(x_v)\gP_L(x_l)} \bigg)}_{const} 
        -2\E_{x_v,x_l} f_V(x_v)^\top f_L(x_l) \\
        &~~~~~~~~~~~+ \E_{x_v^-, x_l^-}\left(f_V(x_v^-)^\top f_L(x_l^-) \right)^2\\
        =& \gL_{\rm SCL}(f_V,f_L) + const,
\end{align*}
which completes the proof.
\end{proof}
Theorem \ref{thm:spectral=asymetric} reveals a crucial fact that multi-modal contrastive learning essentially learns the low-rank factorization of the co-occurrence matrix. Meanwhile, we notice that the original factorization loss is actually intractable to directly solve because of the exponentially large size of the co-occurrence matrix $P_M$, while multi-modal contrastive learning avoids this problem by transforming it into a tractable and scalable objective that simply requires samples from the joint probability $\gP_M$. But theoretically, this equivalence allows us to characterize the overall distribution of multi-modal contrastive learning, and provides guarantees on downstream tasks for its ideal representations in the following part.

\subsection{Characterizing Ideal Representations of Multi-modal Contrastive Learning}
\label{sec:ideal-property}
In multi-modal contrastive learning (MMCL) like CLIP \cite{clip}, a common pipeline is to apply the pretrained representations to downstream visual tasks like image classification. Therefore, in order to characterize the pretraining and downstream behaviors of MMCL, it matters for us to understand the properties of the optimally pretrained representations, and how they generalize to downstream tasks. 

\textbf{Ideal Representations.} First, we characterize the \textit{general solution} to the multi-modal pretraining loss, under the ideal assumption that the neural networks are expressive enough.
\begin{theorem}
Let $\tilde{P}_M=U\Sigma V^\top$ is the singular value decomposition (SVD) of the normalized co-occurrence matrix $\tilde{P}_M$ (Eq.~\ref{eqn:normalization}), where $U\in\sR^{N_V\times r}, V\in\sR^{r\times N_L}$ are unitary matrices, and $\Sigma=\diag(\sigma_1,\dots,\sigma_r)$ contains descending singular values $\sigma_1\geq\dots\sigma_r\geq0,,r=\min(N_V,N_L)$. 
Assume the neural networks are expressive enough for any features. The multi-modal contrastive loss (Eq.~\ref{eqn:spectral loss}) attains its optimum when $\forall\ x_v\in\gX_V,x_l\in\gX_L$,
\begin{subequations}
\begin{align}
f^*_V(x_v)&=\frac{1}{\sqrt{\gP_V(x_v)}}\left(U^k_{x_v}DR\right)^\top, \label{eqn:optimal-V}\\
f^*_L(x_l)&=\frac{1}{\sqrt{\gP_L(x_l)}}\left(V^k_{x_l}\diag({\sigma_1},\dots,{\sigma_k})D^{-1}R\right)^\top, \label{eqn:optimal-L}
\end{align}
\label{eqn:optimal-encoders}
\end{subequations}
where $U_{x}$ takes the $x$-th row of $U$, and $U^k,V^k$ denote the submatrices containing the first $k$ columns of $U,V$, respectively;
$D\in\gR^{k\times k}$ is an arbitrary invertible diagonal matrix; and $R\in\sR^{k\times k}$ is an arbitrary unitary matrix. 
\label{thm:optimal-representation}
\end{theorem}

Theorem \ref{thm:optimal-representation} shows that the ideal representations of MMCL are largely determined by the $k$ leading eigenvectors, up to some affine transformations (scaling $D$ and rotation $R$). Although the optimal solution is not unique, when we apply this representation to the linear probing task, the linear classifier can absorb the differences in affine transformations and yield the same classification error for different variants at the optimum. 
Built upon these optimal representations, we are ready to establish formal guarantees for the generalization of multi-modal contrastive learning on the downstream linear probing tasks in both the visual and language domains.

\begin{theorem}
Given a specific joint data distribution $\gP_M$, we define the labeling error $\alpha$ as the average label agreement among the visual-language  positive pairs $(x_v,x_l)\sim\gP_M$, \ie 
\begin{equation}
    \alpha=\mathbb{E}_{x_v,x_l} \mathbbm{1}[{y}(x_v)\neq y(x_l)],
\end{equation}
where ${y}(\cdot)$ returns the ground-truth label of the operand. 
Denote the empirical estimate of the visual and text encoders from $n$ pretraining examples as $\hat{f}^*_V,\hat{f}^*_L$, respectively. With probability $1-\delta$, the visual linear probing  error $\mathcal{E}(\hat{f}_V^*)$ and text linear probing error $\mathcal{E}(\hat{f}_L^*)$ can be upper-bounded by
\begin{equation}
\begin{aligned}
\big\{\mathcal{E}(\hat{f}_V^*),&\mathcal{E}(\hat{f}_L^*)
\big\} \lesssim \frac{\alpha}{1-\sigma_{k+1}^2}\\
&+ \underbrace{\frac{ck}{\Delta^2_\sigma}\left(\widehat{\gR}_{n/3}(\gF) + \sqrt{\frac{\log 2/\delta}{2n/3}} + \delta \right)}_\text{finite-sample generalization terms}    
\end{aligned}
    \label{eqn:generalization-bound}
\end{equation}
where $\lesssim$ omits some constant terms, $\sigma_{k+1}$ (c.f. Theorem \ref{thm:optimal-representation}) is the $(k+1)$-th largest singular value of the normalized co-occurrence matrix $\tilde{P}_M$. 
In the finite-sample generalization terms, $\hat{\gR}_{n/3}(\gF)$ denotes a Rademacher complexity of the model class $\gF$ with $n/3$ samples, $k$ is the representation dimension, $\Delta_\sigma = \sigma^2_{\floor{3k/4}}-\sigma^2_{k}$, and $c \lesssim (k\kappa + 2k\kappa^2 + 1)^2$ with $\kappa$ upper bounding $\Vert f_V(x) \Vert _\infty$ and $\Vert f_L(x) \Vert _\infty$.  
\label{thm:downstream performance}
\end{theorem}
In the upper bound of Eq.~\ref{eqn:generalization-bound}, aside from the canonical generalization terms relating to the number of samples and neural network complexity, there are two important factors reflecting the influence of the multi-modal pretraining task, the labeling error $\alpha$ and the singular value $\sigma_{k+1}$. 

\textbf{Labeling error $\alpha$} accounts for the label mismatch between the constructed visual-language pairs, which may differ in practice depending on how the dataset is constructed. For example, the MS-COCO dataset contains human-provided captions for 120K images using Amazon Mechanical Turk \cite{mscoco}, while the large-scale YFCC dataset \cite{thomee2016yfcc100m} contains 99M Flickr images along with their posted titles as captions without filtering or post-processing, which could be quite noisy. A recent work \cite{santurkar2022caption} empirically finds that a single MS-COCO image-caption pair is worth five YFCC captions for CLIP training. These findings can be justified by our theory that the written captions in MS-COCO induce a smaller labeling error $\alpha$.

\textbf{Singular value $\sigma_{k+1}$} is a spectral property of the co-occurrence matrix ${P}_M$. One way to understand its role is from a graph perspective. Specifically, we can regard $P_M$ as a (partial) adjacency matrix of a bipartite graph\footnote{For a bipartite graph, only interleaving edges between $\gX_V$ and $\gX_L$ (represented by ${P}_M$) could contain non-zero weights. So we consider ${P}_M$ for simplicity.} established between the visual set $\gX_V$ and the language set $\gX_L$. According to the spectral graph theory \cite{chung1997spectral}, the singular values generally represent the connectivity of the bipartite graph (\eg how many disjoint sub-graphs), and smaller leading singular values correspond to better connectivity (\eg fewer sub-graphs). Therefore, Theorem \ref{thm:downstream performance} shows that better connectivity (by creating diverse connections between samples) with a smaller $\sigma_{k+1}$ could bring smaller downstream errors. In fact, several recent works can be understood as increasing the diversity of multi-modal pairs by data augmentations. For example, FLIP \cite{li2022scaling} introduces patch masking to the images input, and \citet{santurkar2022caption} rewrite text captions using a GPT model. Our generalization bound provides a theoretical justification for the effectiveness of these approaches.

To warp up, our generalization bounds in Theorem \ref{thm:downstream performance} provide not only guarantees but also principled guidelines for multi-modal contrastive learning: 1) we should create high-quality multi-modal pairs by human writing or automatic filtering to reduce the labeling error $\alpha$, and 2) we should create better multi-modal diversity by data augmentations in both domains to ensure a smaller singular value $\sigma_{k+1}$.

\subsection{Discussion}
In this section, we establish the first comprehensive study on the theoretical guarantees of multi-modal contrastive learning in terms of two aspects: optimal representations and downstream guarantees. A closely related work is \citet{haochen} that establishes theoretical guarantees for self-supervised contrastive learning. Our analysis extends their theory to the multi-modal setting, with the following key differences:
\begin{itemize}
    \item[1)] Data generation. Their analysis only applies to positive pairs $(x,x^+)$ that are both augmented samples from the same domain $\gX$, while the multi-modal pair $(x_v,x_l)$ are directly given by data samples and are asymmetric ones from \emph{different domains $\gX_V,\gX_L$}. Correspondingly, our analysis deals with the multi-modal co-occurrence matrix $\tilde{P}_M$ instead of the aggregated augmentation graph $\tilde{A}$ defined over $\gX$ in \citet{haochen} as the approximation target.

    \item[2)] Learning objective. Their analysis only applies to the uni-model spectral contrastive loss using a Siamese architecture, which corresponds to \emph{symmetric} matrix factorization. Instead, in multi-modal learning, the positive pairs are not symmetric and require different encoders in general. Correspondingly, we propose the multi-modal spectral contrastive loss that corresponds to \emph{asymmetric} matrix factorization, which requires different techniques to analyze and yield different optimal representations and downstream generalization bounds. 

\end{itemize}

\begin{figure*}[t]
    \centering
    \subfigure[SimCLR]{
    \includegraphics[width=.3\textwidth]{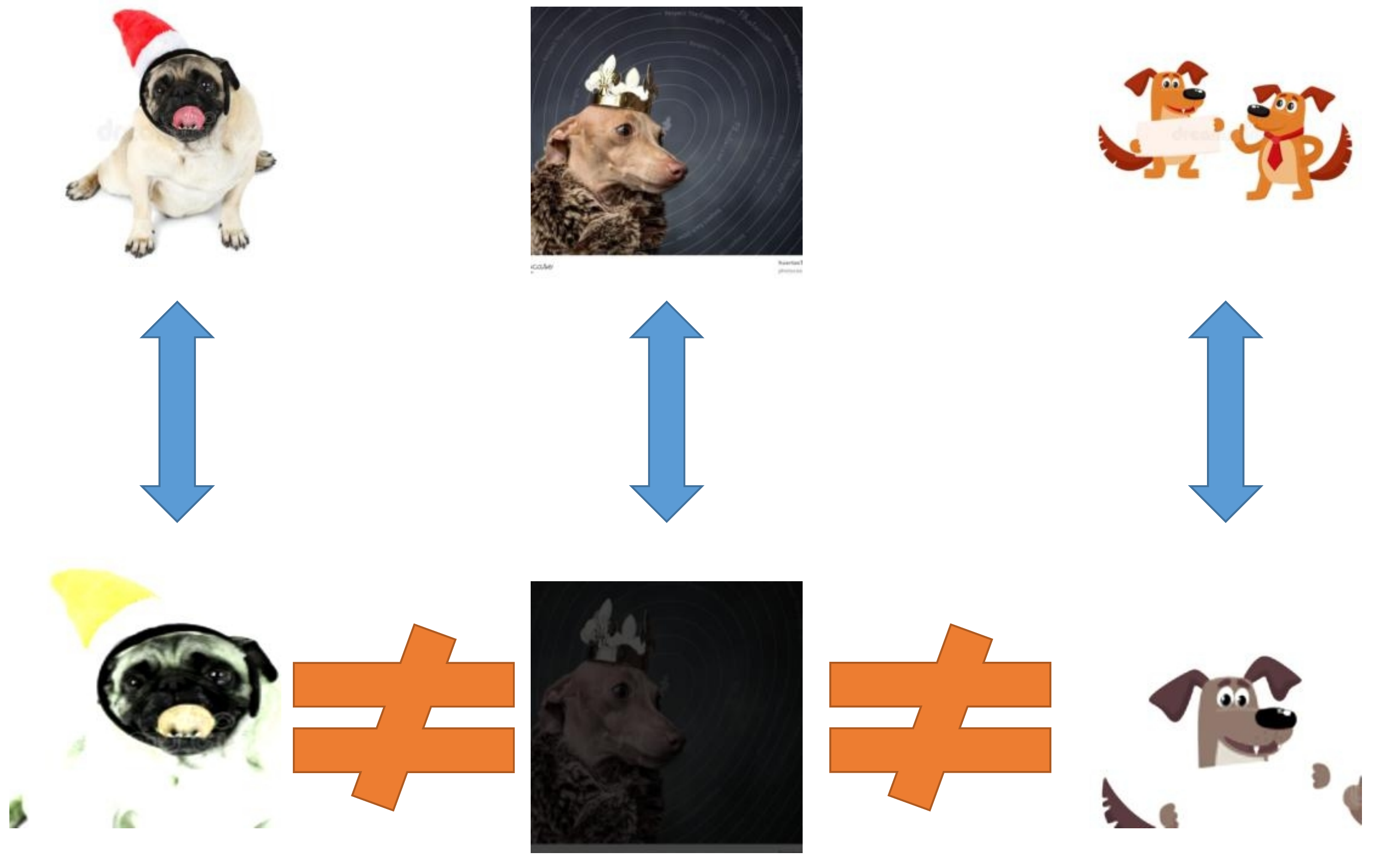}
    \label{fig:hierarchical simclr}
    }    
    \subfigure[CLIP]{
    \includegraphics[width=.3\textwidth]{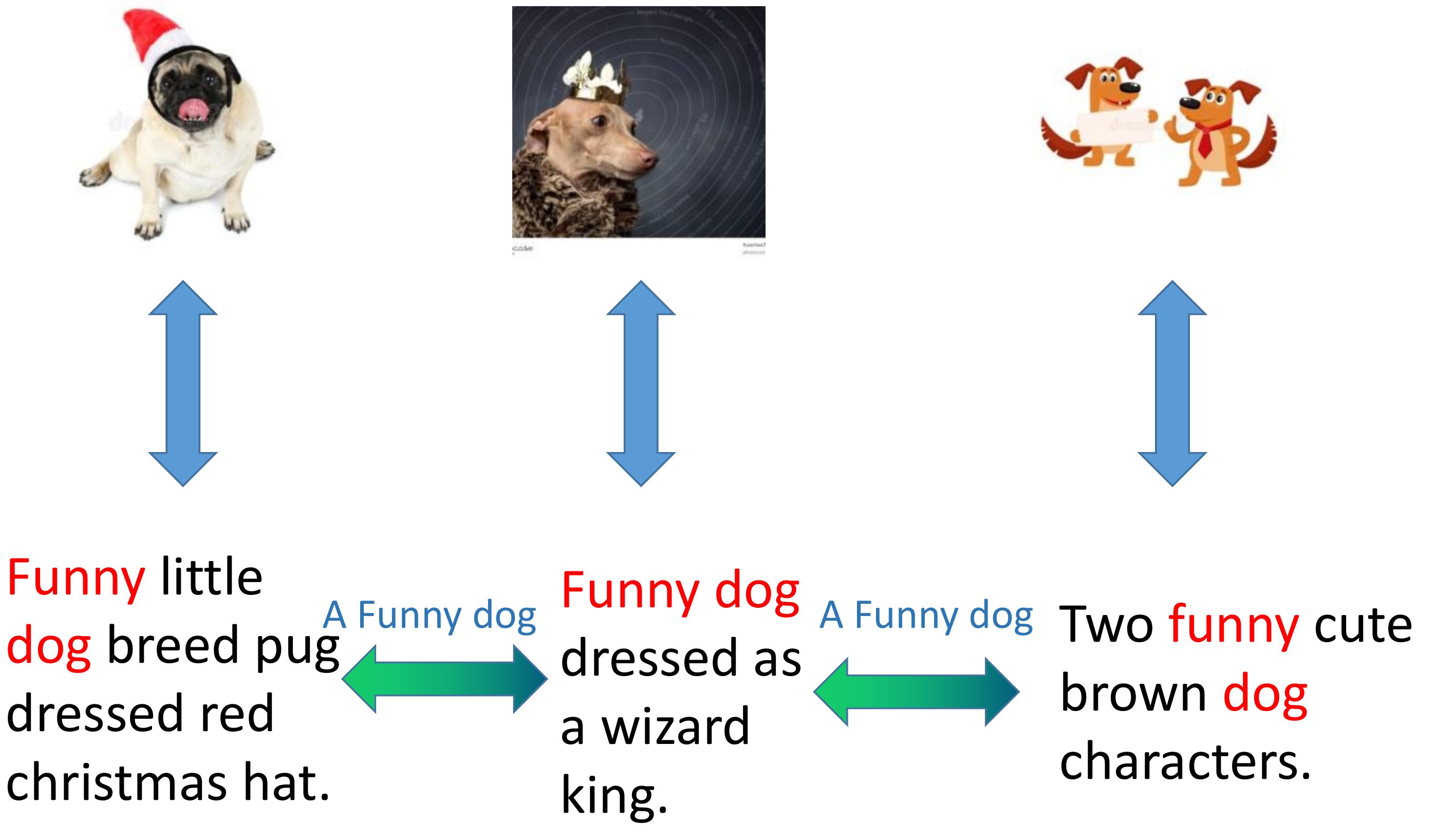}
    \label{fig:hierarchical clip}
    }
    \caption{Illustration of raw and augmented samples generated by SimCLR and CLIP on the CC12M dataset \cite{cc12m}, where the former are generated by manual data augmentations and the latter are induced by visual-language pairs.}
    \label{fig:hierarchical samples}
\end{figure*}
\section{Formal Comparison between Multi-modal and Self-Supervised Contrastive Learning}
\label{sec:comparison}

In Section \ref{sec:downstream guarantees}, we have established a theoretical framework for analyzing multi-modal contrastive learning (MMCL) from the perspective of asymmetric matrix factorization. Meanwhile, we know that MMCL originates from self-supervised contrastive learning (SSCL) like SimCLR \cite{simclr} and MoCo \cite{moco}, which is self-supervised (usually visual). These two contrastive learning paradigms have a close resemblance by both adopting InfoNCE-like objectives, while they differ mainly on the chosen positive and negative pairs. Take two representative methods in each paradigm, CLIP (MMCL) and SimCLR (SSCL), as an example. CLIP adopts visual-language pairs collected from the Internet, while SimCLR generates positive pairs by visual data augmentations like cropping and color jittering. Despite the similarity in learning objectives, CLIP shows much better performance on zero-shot and few-shot transfer learning tasks than SimCLR \cite{clip}, suggesting that different sources of positive pairs have a crucial impact on the downstream performance of contrastive learning. Nevertheless, there still lack theoretical understanding and characterization of this phenomenon.

In this section, we propose a unified theoretical framework to understand the inherent connections between the two paradigms (Section \ref{sec:unified-perspective}). Based on this unified perspective, we compare CLIP and SimCLR on real-world data to understand their differences in downstream tasks (Section \ref{sec:unified-perspective}). At last, we theoretically analyze the differences from a data generation perspective (Section \ref{sec:heirarchical-random-graph}).

\subsection{Unified Formulation and Analysis for Multi-modal and Self-Supervised Contrastive Learning 
}
\label{sec:unified-perspective}
We begin with a brief introduction to self-supervised contrastive learning. Instead of using raw images $x_v\in\gX_V$ as in multi-modal contrastive learning, self-supervised contrastive learning like SimCLR \cite{simclr} applies aggressive data augmentation $\gA(\cdot|x_v)$ two times and get a pair of augmented samples $x_a,x_a^+\in\gX_A$ as positive pairs to align together.
Accordingly, the negative sample is defined as augmented samples $x^-_a$ independently drawn from its marginal distribution. The self-supervised spectral contrastive loss \cite{haochen} learns a Siamese visual encoder $f_V:\gX_A\to\sR^k$ with 
\begin{equation}
\begin{aligned}
\gL^{\rm ss}_{\rm SCL}(f_V)=&-2\mathbb{E}_{x_a,x^+_a}f_V(x_a)^\top f_V(x^+_a)\\
&~~+\mathbb{E}_{x_a,x^-_a}(f_V(x_a)^\top f_V(x_a^-))^2,
\end{aligned}
\end{equation}
where the joint distribution of positive pairs follows
\begin{equation}
\gP_A(x_a,x_a^+)=\E_{x_v\sim\gP_V}\gA(x_a|x_v)\gA(x'_a|x_v),
\end{equation}
which is marginalized over the augmentations of all natural samples. 
Different from multi-modal learning, the joint distribution is symmetric, \ie $\gP_A(x_a,x_a^+)=\gP_A(x^+_a,x_a)$, and \citet{haochen} show that this self-supervised loss is equivalent to a symmetric matrix factorization (SMF) objective. Nevertheless, there is a noticeable difference between the multi-modal and self-supervised objectives, that the joint distribution $\gP_M$ defines connections between two domains $\gX_V,\gX_L$  while $\gP_A$ defines connections only among visual samples in $\gX_A$. It thus remains unclear to us how to compare the quality of multi-modal and self-supervised pairs and characterize their influence on downstream tasks.

A key insight here: we notice that CLIP does not only work well for multi-modal tasks like image-text retrieval, but also performs surprisingly well on visual-only tasks like zero-shot image classification, which indicates that it also implicitly aligns semantically similar visual samples together during the joint embedding process. The following theorem characterizes this intuition by establishing an equivalence between multi-modal contrastive learning and a corresponding self-supervised contrastive learning objective among visual-only samples.
\begin{theorem}
The optimal visual representations of multi-modal contrastive learning (Eq.~\ref{eqn:optimal-V}) are equivalent (up to scaling and rotation) to that of the following uni-modal contrastive learning objective,
\begin{equation}
\begin{aligned}
\gL^{\rm uni}_{\rm SCL}(f_V)=&-2\mathbb{E}_{x_v,x^+_v}f_V(x_v)^\top f_V(x^+_v)\\
&~~+\mathbb{E}_{x_v,x^-_v}(f_V(x_v)^\top f_V(x_v^-))^2,
\end{aligned}
\label{eqn:self-supervised-clip}
\end{equation}
where $(x_v,x_v^+)$ are drawn from the text-induced joint distribution over visual samples $\gP_T$ that $\forall\ x_v,x'_v\in\gX_V$,
\begin{equation}
 \gP_T(x_v,x'_v)=\E_{x_l\sim \gP_L}\gP_M(x_v|x_l)\gP_M(x'_v|x_l),
\end{equation}
with $\gP_M(x_v|x_l)=\gP_M(x_v,x_l)/\gP_L(x_l)$,
and $x_v^-$ is independently drawn from $\gP_V$.
Accordingly, the linear probing error $\gE(f_V^*)$ of multi-modal learning is also equal to that of the self-supervised learning in Eq.~\ref{eqn:self-supervised-clip}.
\label{thm:multi-uni}
\end{theorem}
Theorem \ref{thm:multi-uni} draws an inherent connection between multi-modal contrastive learning (MMCL) and self-supervised contrastive learning (SSCL) by showing that MMCL also implicitly performs uni-modal contrastive learning among visual samples, just like SSCL. Notably, different from SSCL that relies on manual data augmentations $\gA(x_a|x_v)$, MMCL's uni-modal objective (Eq.~\ref{eqn:self-supervised-clip}) leverages the multi-modal conditional distribution $\gP_M(x_v|x_l)$ to generate positive visual pairs \emph{via languages as a pivot}. In other words, the multi-modal signals serve as a new type of data augmentation such that image pairs $x_v,x_v^+$ with the same (or similar) text descriptions can serve as positive pairs for uni-modal contrastive learning, as illustrated in Figure \ref{fig:hierarchical clip}.

This unified perspective enables us to understand the advantage of CLIP over SimCLR for visual representation learning \cite{clip}. Intuitively, compared to SimCLR relying on object-agnostic and low-level manual data augmentations, \eg color and contrast variation in Figure \ref{fig:hierarchical simclr}, text descriptions contain high-level semantics of images (\eg ``funny'', ``dog'' in Figure \ref{fig:hierarchical clip}), and the use of the text-induced augmentation in CLIP can bridge semantically similar images more effectively. Thus, CLIP has two main advantages over SimCLR for downstream tasks according to Theorem \ref{thm:downstream performance}. First,  CLIP has a lower labeling error because the text-induced positive pairs usually contain the same object and while manual data augmentations often lose the object. Second, CLIP yields better connectivity among visual samples using high-level semantics. In the following, we provide empirical and theoretical comparisons to characterize the differences between them.

\begin{table}[]
\centering
\caption{
Comparison (in the uni-model setting) of estimated labeling error and intra-class connectivity between CLIP and SimCLR.}
\begin{tabular}{@{}lcc@{}}
\toprule
                         & CLIP           & SimCLR \\ \midrule
Labeling Error ($\downarrow$)          & \textbf{0.601} & 0.846 \\
Intra-class Connectivity ($\uparrow$)  & \textbf{1.322} & 1.072  \\ \bottomrule
\end{tabular}
\label{tab:empirical results and lb and it}
\end{table}

Based on the unified theoretical understanding above, we further investigate the differences between the augmentation-induced joint distribution $\gP_A$ (self-supervised, SimCLR) and the text-induced one $\gP_T$ (multi-modal, CLIP) on real-world data. For a fair comparison, we pretrain the same backbone ViT-B \cite{dosovitskiy2020image} on the same dataset, YFCC15M \cite{thomee2016yfcc100m,clip}, and evaluate the learned representations on ImageNet \cite{deng2009imagenet}. For efficiency, we randomly draw 1,000 samples from 10 random classes of the ImageNet validation set.
According to the matrix factorization perspective, the learned features approximate the ground-truth distribution (unknown to us). Thus, we can approximately calculate the (uni-modal) labeling error and sample connectivity using learned representations. For an intuitive measure of the desired sample connectivity, we calculate the average feature similarity between intra-class samples as a surrogate metric. See details in Appendix \ref{sec:comparison-details}.

From Table \ref{tab:empirical results and lb and it}, we observe that the labeling error of SimCLR is indeed much larger than that of CLIP (0.846 \textit{v.s.~}0.601), suggesting that the text-induced (implicit) positive images have higher semantic consistency than manual image transformations. Meanwhile, we also observe that CLIP has high intra-class connectivity than SimCLR (1.322 \textit{v.s.}~1.072), suggesting that text descriptions can induce better intra-class sample diversity with the high-level semantic relationship.

\begin{figure*}[t]
    \centering
    \subfigure[an illustration of the hierarchical structure]{
    \includegraphics[width=.35\textwidth]{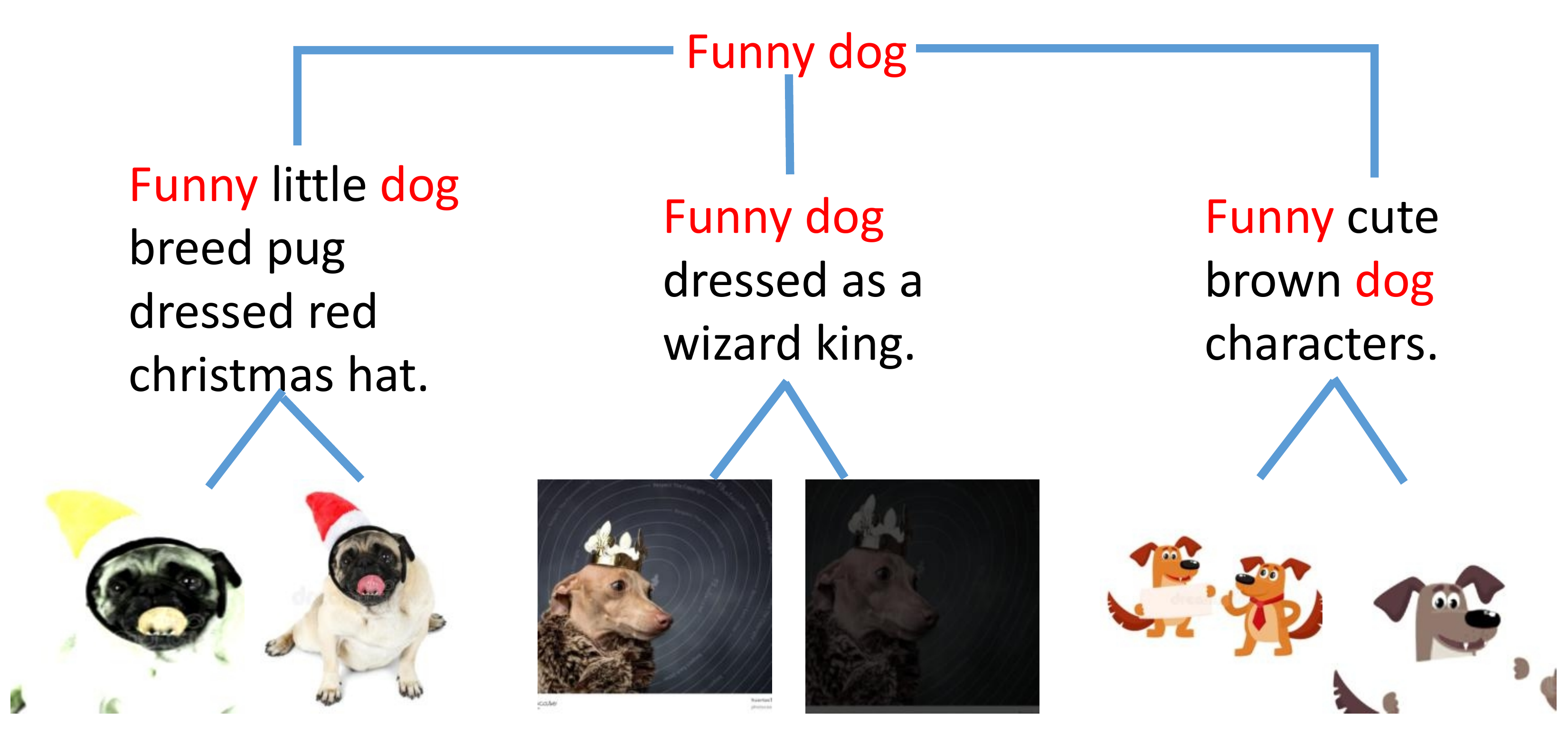}
    \label{fig:hierarchical-structure-example}
    }    
    \subfigure[a hierarchical random graph]{
    \includegraphics[width=.3\textwidth]{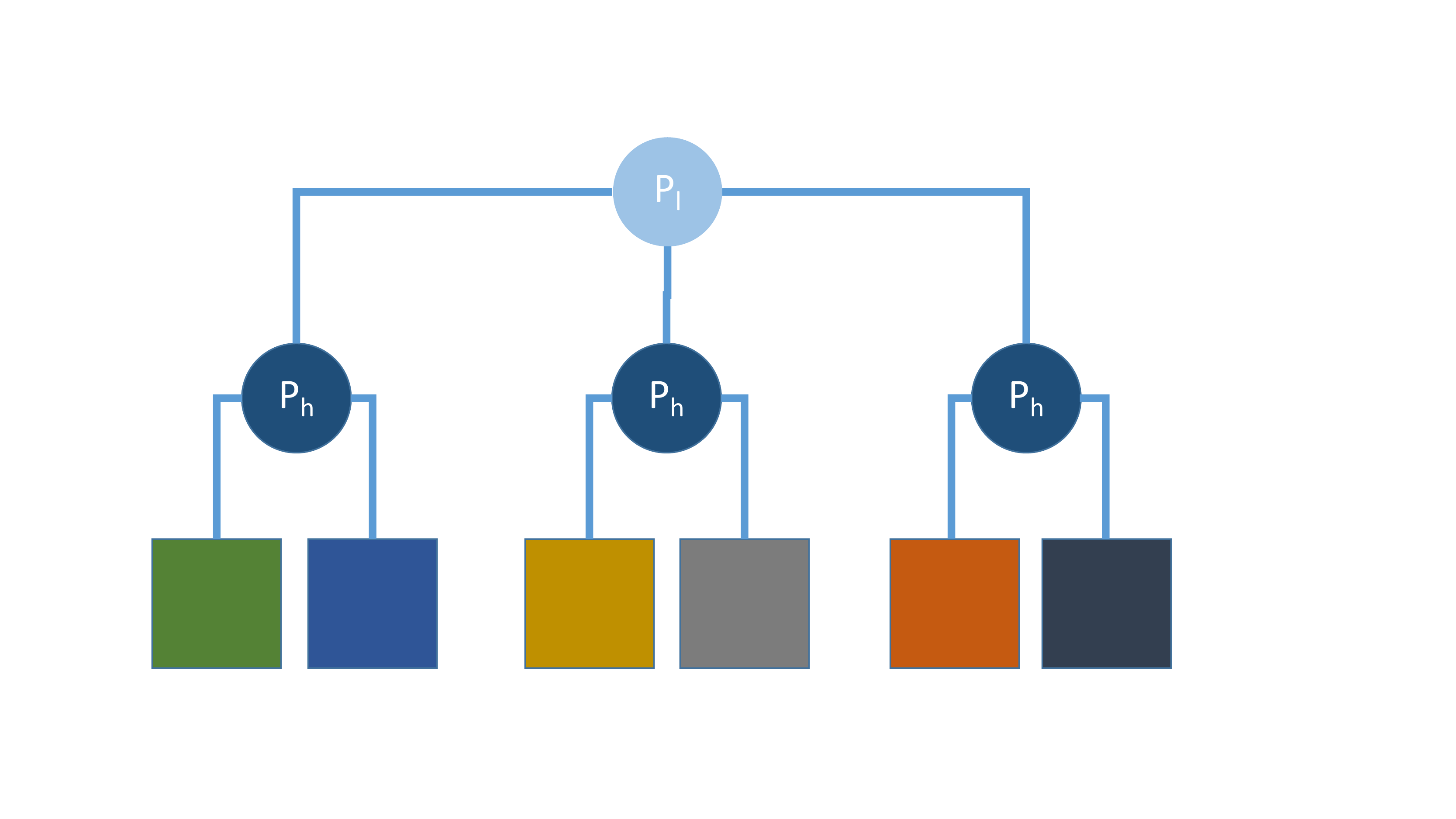}
    \label{fig:hierachical-random-graph}
    }
    \caption{Illustrations of the hierarchical structure on real-world datasets CC12M and a hierarchical random graph with two hidden layers. Here, each internal node is associated with a probability that a pair of vertices in the left and right subtrees of that node are connected.}
\label{fig:hi visualization}
\end{figure*}

\subsection{
A Data Generation Perspective via the Language of Hierarchical Random Graph}
\label{sec:heirarchical-random-graph}

As discussed above, the key difference between augmentation and text-induced positive pairs is that they operate on different levels of semantics.  This difference can be understood and modeled in a hierarchical structure of data generation. As shown in the examples in Figure \ref{fig:hierarchical-structure-example}, we can regard that the three images of funny dogs are firstly generated under high-level concepts captured by their text description, and then adding more detailed variations that can be captured by data augmentations. Therefore, the shared text span ``funny dog'' can draw these images together, but the commonly used data augmentations cannot because they are very different in pose and style.

Inspired by the observation that the joint distribution between positive visual pairs $\gP_T(x_v,x_v')$ can be regarded as the adjacency matrix of a graph over all image samples \cite{haochen}, we model this distribution (graph) with hierarchical random graph \cite{clauset2008hierarchical} designed to model the hidden structure of a given graph. Different from vanilla random graph where each edge is randomly drawn with the same probability, hierarchical random graph assumes that the edges are drawn according to a hierarchical tree, which suits our need to characterize different levels of semantics. In a hierarchical random graph $\gG$ shown in Figure \ref{fig:hierachical-random-graph}, each internal node $s$ is associated with a probability $p_s$, each leaf node is a node in the original graph, and the probability of having an edge between two nodes is the probability contained in their lowest common ancestor node. In our case, we assume two hidden layers for simplicity, with $p_l$ modelling the probability high-level connection in the first layer and $p_h$ modeling the probability of lower-level connection in the second layer. We assume $p_h>p_l$ as there are less high-level interactions between samples.

The following theorem shows that a larger high-level connection probability $p_l$ yields better downstream performance by inducing better graph connectivity (algebraically measured by the singular value $\sigma_{t}$).
\begin{theorem}
    For two three-layer hierarchical random graphs $\mathcal{G}$, $\mathcal{G}'$ with probabilities $(p_l, p_h),(p'_l, p'_h)$, respectively. If $p_h-p_l \leq p'_{h}-p'_{l}$, we have
    \begin{align*}
       \sigma_t \leq \sigma_t',
    \end{align*}
    where the $\sigma_t, \sigma'_t$ are the $t$-th largest singular values of $\gG,\gG'$, respectively.
    According to Theorem \ref{thm:downstream performance}, smaller singular value indicates better downstream performance under the same labeling error $\alpha$. Therefore, contrastive learning with samples generated according to graph $\gG'$ will have better downstream performance.
    \label{thm:eigenvalues of hierarchal graph}
\end{theorem}
Theorem \ref{thm:eigenvalues of hierarchal graph} shows smaller $p_h-p_l$ can bring better downstream generalization\footnote{We note that two quantities $p_h,p_l$ are not independent. Since the total probability sums to one, a higher $p_h$ means a lower $p_l$, and vice versa.}. In practice, we can improve $p_h$ by generating positive samples sharing common high-level semantics, as done in CLIP with the text description of the image. Therefore, our hierarchical random graph perspective can help characterize the benefit of CLIP over SimCLR from the kind of information they leverage. This perspective also suggests a way to improve (self-supervised) contrastive learning, that is to add more diverse with better augmentation strategies, such as, using realistic generative models like diffusion models \cite{ho2020denoising}. In the next section, we provide empirical verification of this understanding by showing how MMCL information can be used to boost augmentation-based SSCL methods like SimCLR.
\begin{table*}[h]
\centering
\caption{The linear probing accuracy of SimCLR and its CLIP-guided variants on ImageNet (ViT-B, 100-epoch training).}
\begin{tabular}{@{}lccccc@{}}
\toprule
Method          & Baseline (SimCLR) & AddNewPositive  & DropFalsePositive & DropFalseNegative & DropEasyNegative \\ \midrule
Linear Acc & 61.2     & \textbf{67.4 (+6.2)} & 61.8 (+0.6)                   & 61.4 (+0.2)                   & 62.3 (+1.1)                \\    \bottomrule
\end{tabular}
\label{tab:linear-probing}
\end{table*}  

\section{Boosting SimCLR with CLIP-Guided Resampling
}
\label{sec:verification}

Learning from the theoretical and empirical evidence in Section \ref{sec:comparison}, we have known that compared to self-supervision, languages are better at generating positive pairs for visual representation learning due to their advantage of capturing high-level similarities. In this section, we further leverage this advantage to improve self-supervised learning. 

Prior to ours, there are several papers exploring the combination of self-supervision and multi-modal supervision, such as SLIP \citep{mu2022slip}, DeCLIP \cite{declip}, and FLIP \cite{li2022scaling}. Contrary to these methods all focusing on pretraining on multi-modal data, in this work, we focus on utilizing the estimated multi-modal information in a pretrained CLIP model to improve self-supervised contrastive learning (SimCLR) from unlabeled images alone, which, up to our knowledge, is not considered yet. Our experiment is designed as a verification of our analysis above, because if the language information is as helpful for uni-modal contrastive learning as we suppose, the CLIP-guided SimCLR can obtain better performance on downstream tasks.

\subsection{Methods} 
Following our analysis, we consider four CLIP-guided resampling strategies for leveraging CLIP to help self-supervised contrastive learning with SimCLR. 
    
\textbf{AddNewPositive} \& \textbf{DropFalsePositive}. Because multi-modal contrastive learning is good at generating more diverse and consistent positive pairs (Figure \ref{fig:hierarchical clip}), we leverage the pretrained CLIP to generate a new pair of positive samples for training SimCLR. Specifically, in a mini-batch, we find the nearest neighbor of each sample $x$ in the feature space of CLIP, denoted as $\gN(x)$, and regard $(x,\gN(x))$ as a pair of positive samples. We mix this new positive pair with the original self-supervised one with a tunable ratio. On the other hand, because multi-modal pairs have less labeling error (Table \ref{tab:empirical results and lb and it}), CLIP can also be leveraged to filter out false positive pairs that may contain different objects with a tunable ratio. 
    
\textbf{DropFalseNegative} \& \textbf{DropEasyNegative}.  We can also leverage CLIP to select negative samples. One option is to drop negatives with the largest similarity, which could be false negatives from the same class of positive samples. Another is to drop negative samples with the smallest similarity, with corresponds to easy negative samples that are already pushed apart.

For the CLIP model, we adopt the pretrained ViT-B provided by the official implementation. For SimCLR, following the standard protocol, we pretrain a ResNet-50 \cite{he2016deep} on ImageNet for 100 epochs. See details in Appendix \ref{sec:verification-details}.

\subsection{Results} 
From Table \ref{tab:linear-probing}, we can see that all four techniques can bring benefits over the vanilla SimCLR, suggesting that the multi-modal information in CLIP indeed benefits self-supervised learning in terms of both positive and negative sample selection. Meanwhile, comparing the four strategies, we notice that AddNewPositive with CLIP brings the highest improvement of 6.2\% accuracy over the vanilla SimCLR. This successfully verifies our previous analysis that multi-modal learning is better at generating diverse and positive samples than self-supervised learning for better downstream performance. We leave more advanced techniques for leveraging this observation for future work.

\section{Conclusion}
In this paper, we proposed the first theoretical framework for multi-modal contrastive learning. By drawing the connection to asymmetric matrix factorization, we characterized its optimal representations and established the first guarantees on the downstream generalization of multi-modal contrastive learning. Based on our framework, we provided a unified perspective of multi-modal and self-supervised contrastive learning, characterized their differences on real-world data, and verified our insights by bringing benefits on benchmark datasets. In this way, our theory has established a principled understanding of multi-modal contrastive learning, while delivering practical insights for combining multi-modal and self-supervised learning methods.

\section*{Acknowledgement}
Yisen Wang is partially supported by the National Key R\&D Program of China (2022ZD0160304), the National Natural Science Foundation of China (62006153), Open Research Projects of Zhejiang Lab (No. 2022RC0AB05), and Huawei Technologies Inc.

\bibliography{example_paper}
\bibliographystyle{icml2023}

\newpage
\appendix
\onecolumn

\section{Experimental Details}

\subsection{Details of Empirical Comparison in Section \ref{sec:unified-perspective}}
\label{sec:comparison-details}

\textbf{Approximation of Data Probability.} Similar to the multi-modal spectral loss,  Equation (\ref{eqn:self-supervised-clip}) can be rewritten as a matrix decomposition loss, i.e., $\gL_{\rm SCL}^{\rm uni}(f_V) = \Vert \tilde{P}_T - F_V F_V^\top \Vert ^2 +const$, where $P_T$ is the co-occurrence matrix of the distribution $\gP_T(x_v,x_v')$, $(\tilde{P}_T)_{(x_v,x_v')} = \frac{\gP_T(x_v,x_v')}{\sqrt{\gP_V(x_v)\gP_V(x_v')}} $ and $(F_V F_V^\top)_{(x_v,x_v')} = \frac{f_V(x_v)^\top f_V(x_v')}{\sqrt{P_V(x_v)P_V(x_v')}} $. So $(P_T)_{(x_v,x_v')}$ can be approximated by $f_V(x_v)^\top f_V(x_v')$ when the loss is minimized. Similarly, we can estimate the co-occurrence matrix $(P_{A})_{(x_a,x_a^+)}$ of $\gP_A(x_a,x_a^+)$ by $f_V(x_a)^\top f_V(x_a^+)$. In practice, we use ViT-Base trained by CLIP \citep{clip} and SimCLR \citep{simclr}  as the encoders.

\textbf{Setup.} We respectively 
 encode the samples from 1000 samples randomly selected from 10 classes of ImageNet \citep{imagenet} with two encoders and construct the embedding matrix $\hat{F}_T \in \mathbb{R}^{1000\times k}$ and $\hat{F}_A \in \mathbb{R}^{1000\times k}$ ($k$ is the output dimension of ViT-Base)\footnote{As the samples of the $P_A$ are augmented images, we transform the selected samples with the augmentations used in SimCLR when constructing $\hat{F}_A$.}. Then we normalize the similarity matrices of the embeddings to and estimate the co-occurrence matrices with them, i.e., $\hat{P}_{T} = \operatorname{normalize}(\hat{F}_T \hat{F}_T^\top)$, $\hat{P}_{A} = \operatorname{normalize}(\hat{F}_A \hat{F}_A^\top)$. In the next step, we evaluate the properties of the estimated matrices, e.g., the labeling error, the eigenvalues,  etc. 

 \textbf{Estimation of Labeling Error.} When evaluating the labeling error $\alpha$ in Eq. \ref{eqn:generalization-bound}, as ImageNet is a vision dataset, we have no access to the corresponding text data. So we use a surrogate metric $\alpha_{T}$, and it is defined as:
\begin{equation}
\begin{aligned}
\alpha_{T}=\sum \limits_{x_v,x'_v}(P_T)_{x_v,x'_v} \mathbbm{1}[y(x_v)\neq y(x'_v)],
\end{aligned}
\end{equation}
and $y(x_v)$ denotes the ground-truth label of $x_v$. Note that $\alpha_T$ is lower bounded by the ground-truth labeling error $\alpha$:
\begin{proposition}
For the surrogate metric $\alpha_T$, we have
\begin{align*}
     \alpha \geq \frac{1}{2} \alpha_T.
\end{align*}
\label{pro: estimated labeling error}
\end{proposition}
\begin{proof}
Expanding the estimated labeling error and we obtain
\begin{align*}
\alpha_T &= \sum\limits_{(x_v,x'_v)}\gP_T(x_v,x_v')\mathbbm{1}[y(x_v)\neq y(x'_v)]\\
&=\sum\limits_{x_v,x'_v} \E_{x_l}\left[\gP_M(x_v|x_l)\gP_M(x_v|x_l)\mathbbm{1}[y(x_v)\neq y(x'_v)]\right]\\
&\leq\sum\limits_{x_v,x'_v} \E_{x_l}\left[\gP_M(x_v|x_l)\gP_M(x_v|x_l)(\mathbbm{1}[y(x_v)\neq y(x_l)]+\mathbbm{1}[y(x'_v)\neq y(x_l)])\right]\\
&= 2 \E_{x_l} [\gP_M(x_v|x_l)\mathbbm{1}[y(x_v)\neq y(x_l)]]\\
&= 2 \E_{x_v,x_l} \mathbbm{1}[y(x_v)\neq y(x_l)]\\
&= 2\alpha.
\end{align*}
\end{proof}
As a result, a large $\alpha_T$ implies a large labeling error $\alpha$. Then we replace $P_T$ with $\hat{P}_T$, and obtain the estimation $\hat{\alpha}_{T}=\sum \limits_{x_v,x'_v}(\hat{P}_T)_{x_v,x'_v} \mathbbm{1}[y(x_v)\neq y(x'_v)] $. Similarly, we define the estimated labeling error of $P_A$ as $\hat{\alpha}_{A}=\sum \limits_{x_v,x_v^+}(\hat{P}_A)_{x_v,x_v^+} \mathbbm{1}[y(x_v)\neq y(x_v^+)]$.

\textbf{Estimation of Intra-class Connectivity.} 
When evaluating the intra-class connectivity, we respectively select 1000 samples from 10 different classes of ImageNet. Taking the multi-modal pretraining as an example, following the process we construct $\hat{P}_T$, we respectively construct ten intra-class feature similarity matrices $\{\hat{P}_{in}^k\}_{k=1}^{10}$. Then we randomly select 1000 samples from the selected samples and construct an inter-class feature similarity matrix $\hat{P}_{out}$. We use the average relative value of the intra-class and inter-class feature similarity matrix to represent the intra-class connectivity. To be specific, we denote the intra-class connectivity as $\beta$ and evaluate it by:
\begin{equation}
\begin{aligned}   
   (\hat{P}_{re})_{i,j}^k &= (\hat{P}_{in})_{i,j}^k / \mathop{\rm mean}\limits_{i,j}(\hat{P}_{out}),\\
     \beta_k & = \mathop{\rm mean}\limits_{i,j}(\hat{P}_{re})_{i,j}^k,\\
      \beta &= \mathop{\rm mean}\limits_{k}(\beta_k ).
\end{aligned}
\end{equation}

\subsection{Details of Verification Experiments in Section \ref{sec:verification}}
\label{sec:verification-details}

We use SimCLR \citep{simclr} as our baseline and adopt the popular backbone ResNet-50. With the default setting of SimCLR, we add a projector MLP following the backbone. During the pretraining process of SimCLR, we train the encoder for 100 epochs on ImageNet with 512 batch size and use the LARS optimizer with a cosine annealed learning rate schedule. When estimating the co-occurrence matrix $P_T$, we compute the feature similarity matrix with the well-trained ViT-B encoder provided by the official repository of CLIP \citep{clip}. For selecting new positive pairs, we set the ratio between the new regularizer and the original loss to 1. When filtering false positive samples, we throw the 10\% positive pairs that are most dissimilar in the feature space encoded by the CLIP encoder. And for selecting better negative samples. we respectively throw 5\% samples that have the largest similarity with the positive samples and 10\% samples that have the smallest similarity with the positive samples. After the pretraining process, we train a linear classifier following the frozen backbones and optimize the CrossEntropy loss with the SGD optimizer.

\section{Proofs}

\subsection{Proof of Theorem \ref{thm:spectral=asymetric}}
\begin{proof}
 Expanding the decomposition object $\gL_{\rm AMF}$ and we obtain,
\begin{align*}
\gL_{\rm AMF}(F_V,F_L) &= \Vert \tilde{P}_M - F_VF_L^\top \Vert ^2\\
        &=\sum\limits_{x_v,x_l}\left((\tilde{P}_M)_{x_v,x_l}-(F_V)_{x_v}(F_L)_{x_l}^\top\right)^2\\
        &=\sum\limits_{x_v,x_l}\left( \frac{\gP_M(x_v,x_l)}{\sqrt{\gP_V(x_v)\gP_L(x_l)}}-\sqrt{\gP_V(x_v)}f_V(x_v)^\top \sqrt{\gP_L(x_l)}f_L(x_l)  \right)^2\\
        &=\sum\limits_{x_v,x_l}\left(\frac{\gP_M(x_v,x_l)^2}{\gP_V(x_v)\gP_L(x_l)} +\gP_V(x_v)\gP_L(x_l)\left(f_V(x_v)^\top f_L(x_L)\right)^2 -2\gP_M(x_v,x_l)f_V(x_v)^\top f_L(x_L)\right)\\
        &=\sum\limits_{x_v,x_l}\left(\frac{\gP_M(x_v,x_l)^2}{\gP_V(x_v)\gP_L(x_l)} \right) 
        -2\E_{x_v,x_l} f_V(x_v)^\top f_L(x_l) + \E_{x_v^-, x_l^-}\left(f_V(x_v^-)^\top f_L(x_l^-) \right)^2\\
        &= \gL_{\rm SCL}(f_V,f_L) + const.
\end{align*}
\end{proof}

\subsection{Proof of Theorem \ref{thm:optimal-representation}}
\begin{proof}

    According to Eckart-Young Theorem \citep{eckart1936approximation}, the optimal solution $F_V^\star, F_L^\star$ of the decomposition objective $\gL_{\rm AMF}(F_V,F_L) = \Vert \tilde{P}_M-F_VF_L^\top \Vert ^2$ satisfy:
    \begin{align*}
    F_V^\star (F_L^\star)^\top = U^k \operatorname{diag}(\sigma_1,...,\sigma_k)(V^k)^\top,
    \end{align*}
    where we denote $\tilde{P}_M=U\Sigma V^\top$ as the singular value decomposition of $\tilde{P}_M$, $(\sigma_1,...,\sigma_k)$ are the $k$-largest singular values of $\tilde{P}_M$, the $t$-th column of $U^k \in \mathbb{R}^{N_V \times k}$ contains the corresponding eigenvectors of the $t$-th largest singular values and $V^k  \in \mathbb{R}^{N_L \times k}$ is a unitary matrix. Then we respectively represent the optimal solutions $F_V^\star$ and $F_L^\star$:
    \begin{align*}
    F_V^\star &= U^k D R,\\
    F_L^\star &= V^k \operatorname{diag} (\sigma_1,...,\sigma_k)D^{-1} R,
    \end{align*}
where $R \in \mathbb{R}^{k\times k}$ is a unitary matrix and $D$ is an invertible diagonal matrix. With $(F_V)_{x_v} = (f_V(x_v))^\top \sqrt{P_V(x_v)}$ and $(F_L)_{x_l} = (f_L(x_l))^\top \sqrt{P_L(x_l)}$, we obtain 
\begin{align}
f^*_V(x_v)&=\frac{1}{\sqrt{\gP_V(x_v)}}(U^k_{x_v}DR)^\top, \\
f^*_L(x_l)&=\frac{1}{\sqrt{\gP_L(x_l)}}(V^k_{x_l}\diag({\sigma_1},\dots,{\sigma_k})D^{-1}R)^\top. 
\end{align}   
\end{proof}

\subsection{Proof of Theorem \ref{thm:downstream performance}}
We first introduce a lemma in \citet{haochen}:
\begin{lemma}[Theorem 3.8 in \citet{haochen}]
 Denote the labeling error as $\alpha=\mathbb{E}_{(x_v,x_l)}  \mathbbm{1}[y(x_v)\neq y(x_l)]$. Let $f_V'^\star$ be a minimizer of the $\gL_{\rm SCL}^{\rm uni}(f_V)$, we obtain
\begin{align*}
    \mathcal{E}(f_V'^\star) \leq \frac{2\phi^{y}}{\sigma'_{k+1}}+ 8 \alpha,
\end{align*}
where $\sigma'_{k+1}$ is the k-smallest eigenvalue of the Laplacian matrix of $P_T$.
\label{lem:haochen generalization}
\end{lemma}
Then we give the proof of Theorem \ref{thm:downstream performance} in the following.
\begin{proof}
We denote $y(x)$ as the label of data $x$. Then we define the probability that two image samples related to the same text sample have different labels as 
\begin{equation}
    \phi^{y} = \sum\limits_{x_v,x'_v}\gP_T(x_v,x_v')\mathbbm{1}[y(x_v)\neq y(x'_v)].
\end{equation}
We note that
\begin{align*}
\phi^{y} &= \sum\limits_{(x_v,x'_v)} \gP_T(x_v,x_v')\mathbbm{1}[y(x_v)\neq y(x'_v)]\\
&=\sum\limits_{x_v,x'_v} \E_{x_l}\left[\gP_M(x_v|x_l)\gP_M(x_v|x_l)\mathbbm{1}[y(x_v)\neq y(x'_v)]\right]\\
&\leq\sum\limits_{x_v,x'_v} \E_{x_l}\left[\gP_M(x_v|x_l)\gP_M(x_v|x_l)(\mathbbm{1}[y(x_v)\neq y(x_l)]+\mathbbm{1}[y(x'_v)\neq y(x_l)])\right]\\
&= 2 \E_{x_l} [\gP_M(x_v|x_l)\mathbbm{1}[y(x_v)\neq y(x_l)]]\\
&= 2 \E_{x_v,x_l} \mathbbm{1}[y(x_v)\neq y(x_l)]\\
&= 2\alpha.
\end{align*}
Combined with Lemma \ref{lem:haochen generalization}, we have $\mathcal{E}(f_V'^\star) \leq  \widetilde{O} (\frac{\alpha}{\sigma'_{k+1}})$, where  $\widetilde{O}(\cdot)$ is used to hide universal constant factors. We denote the $(k+1)$-largest singular values of $\tilde{P}_M$ as $\sigma_{k+1}$. As $\tilde{P}_T = \tilde{P}_M\tilde{P}_M^\top$ and the singular values are positive, the $(k+1)$-largest singular values of $\tilde{P}_T$ is $(\sigma_{k+1})^2$, i.e., $\sigma'_{k+1} = 1-(\sigma_{k+1}^2)$.
Combined with Theorem \ref{thm:multi-uni} (proofs are provided in the following), for the image encoder $f_V^\star$ that minimizes $\gL_{\rm SCL}$, we obtain
\begin{equation}
    \mathcal{E}(f_V^\star)=  \mathcal{E}(f_V'^\star) \leq \widetilde{O} (\frac{\alpha}{1-\sigma_{k+1}^2}).
\end{equation}
Obviously, the linear probing error of the text encoder $f_L^\star$ that minimizes $\gL_{\rm SCL}$ has the similar results:
\begin{equation}
    \mathcal{E}(f_L^\star) \leq \widetilde{O} (\frac{\alpha}{1-\sigma_{k+1}^2}).
\end{equation}
Then we consider the empirical loss with finite samples. We construct a multi-modal dataset $\hat{\gX} = \{(z_v^1,z_l^1), ... , (z_v^n,z_l^n)\}$ and the $n$ positive pairs are i.i.d sampled from $\gP_M(x_v,x_l)$. We first sample a permutation $\pi : [n] \rightarrow [n]$, then we construct the positive pairs and negative pairs as follows:
\begin{align*}
x_v^i &= z_v^{\pi(3i-2)},\\
x_l^i &= z_l^{\pi(3i-2)},\\
(x_l^i)^- &= z_l^{\pi(3i-1)},\\
(x_v^i)^- &= z_v^{\pi(3i)}.\\
\end{align*}
and the empirical loss is
\begin{equation}
\begin{aligned}
&\mathcal{L}_{\rm emp}(f_V,f_L) 
&=-\frac{2}{n/3}\sum_{i=1}^{n/3}f_V(x_v^i)^\top f_L(x_l^i)+ \frac{1}{n/3}
\sum_{i=1}^{n/3}(f_V(x_v^i)^\top f_L\left(x_l^i)^{-}\right)^2 + \frac{1}{n/3}
\sum_{i=1}^{n/3}(f_V((x_v^i)^{-})^\top f_L(x_l^i))^2.
\end{aligned}
\end{equation}

Considering the expectation of $\gL_{\rm emp}$, we obtain
\begin{align*}
\mathbb{E}_{\hat{\gX}} \gL_{\rm emp}(f_V,f_L) &=-\frac{2}{n/3}\sum_{i=1}^{n/3}f_V(x_v^i)^\top f_L(x_l^i)+ \frac{1}{n/3}
\sum_{i=1}^{n/3}(f_V(x_v^i)^\top f_L((x_l^i)^{-})^2 + \frac{1}{n/3}
\sum_{i=1}^{n/3}(f_V((x_v^i)^{-})^\top f_L(x_l^i))^2.\\
&= -2\mathbb{E}_{x_v,x_l}f_V(x_v)^\top f_L(x_l) + \mathbb{E}_{x_v\sim \gP_V(x_v), x_l \sim \gP_L(x_l)} (f_V(x_v^i)^\top f_L(x_l^j))^2\\
&= \gL_{\rm SCL}(f_V, f_L).
\end{align*}
So the empirical loss is an unbiased estimator.
We denote that Rademacher complexity of $\gF$ over $n$ data as 
\begin{align*}
\hat{\gR}_n (\gF)= \operatorname{max}_{\{x_1,...x_n\}}\mathbb{E}_\sigma\left[\operatorname{sup}_{f\in \gF,i}\left( \frac{1}{n} \sum_{j=1}^{n} \rho_j f_i(x_j) \right)\right],
\end{align*}
where $f_i(x_j)$ denotes the $i$-th dimension of $f(x_j)$ and $\rho$ is a uniform random vector in $\{-1,1\}^n$.

Following Theorem 4.2 in \citet{haochen}, when $\mathcal{E}(\hat{f}_V^*)$, $\mathcal{E}(\hat{f}_V^*)$ are the minimizers of $\gL_{\rm emp}(f_V,f_L)$, we obtain

\begin{equation}
\begin{aligned}
\big\{\mathcal{E}(\hat{f}_V^*),&\mathcal{E}(\hat{f}_L^*)
\big\} \lesssim \frac{\alpha}{1-\sigma_{k+1}^2}
+ \underbrace{\frac{ck}{\Delta^2_\sigma}\left(\widehat{\gR}_{n/3}(\gF) + \sqrt{\frac{\log 2/\delta}{2n/3}} + \delta \right)}_\text{finite-sample generalization terms}    
\end{aligned}
\end{equation}
where $\lesssim$ omits some constant terms, $\sigma_{k+1}$ (c.f. Theorem \ref{thm:optimal-representation}) is the $(k+1)$-th largest singular value of the normalized co-occurrence matrix $\tilde{P}_M$. 
In the finite-sample generalization terms, $\hat{\gR}_{n/3}(\gF)$ denotes a Rademacher complexity of the model class $\gF$ with $n/3$ samples, $k$ is the representation dimension, $\Delta_\sigma = \sigma^2_{\floor{3k/4}}-\sigma^2_{k}$, and $c \lesssim (k\kappa + 2k\kappa^2 + 1)^2$ with $\kappa$ upper bounding $\Vert f_V(x) \Vert _\infty$ and $\Vert f_L(x) \Vert _\infty$.

\end{proof}

\subsection{Proof of Theorem \ref{thm:multi-uni}}

We first introduce a lemma that states that multiplying the embedding matrix by an invertible matrix on the right will not influence the linear probing error \cite{haochen}:
\begin{lemma}[Lemma 3.1 in \citet{haochen}]
For two learned embedding matrices $F$, $\widetilde{F}$, a diagonal matrix $D$ and an invertible matrix $Q$, if $F = D \widetilde{F} Q$, they have the equal linear probing error, i.e.,
\begin{align*}
\mathcal{E} (F) = \mathcal{E}(\widetilde{F}).
\end{align*}
\label{lem:linear absorbed}
\end{lemma}
Then we give the proof of Theorem \ref{thm:multi-uni} in the following.
\begin{proof} 
   With Theorem \ref{thm:optimal-representation}, the optimal solutions $F_V^\star$, $F_L^\star$ of $\gL_{\rm AMF}(F_V,F_L) = \Vert \tilde{P}_M - F_VF_L^\top \Vert ^2$ can be respectively represented as:
    \begin{align*}
    F_V^\star &= U^kD R,\\
    F_L^\star &= V^k D_2 R,
    \end{align*}
where $R\in \mathbb{R}^{k\times k}$ is a unitary matrix and $D,D_2$ are diagonal matrices that satisfy $D_2 = \operatorname{diag}(\sigma_1,...,\sigma_k)  D^{-1}.$ Following the proof of theorem \ref{thm:spectral=asymetric}, the uni-modal contrastive loss is also equivalent to a matrix decomposition loss, i.e., $\gL_{\rm SCL}^{\rm uni}(f_V) = \Vert \tilde{P}_T - F_V F_V^\top \Vert ^2 +const$, where $(\tilde{P}_T)_{(x_v,x_v')} = \frac{\gP_T(x_v,x_v')}{\sqrt{\gP_V(x_v)\gP_V(x_v')}} $ and $(F_V)_{x_v} = \frac{f_V(x_v)^\top}{\sqrt{P_V(x_v)}}$. Then we consider the objective $L_{\rm mf}(F_V) =\Vert \tilde{P}_T-F_VF_V^\top \Vert ^2 $. Similar to the asymmetric decomposition objective, the optimal solution can be represented as:
\begin{align*}
(F_V^\star)' = U_T^k D_T R_T,
\end{align*}
where $U_T^k \in \mathbb{R}^{N_V\times k}$ contains $k$ corresponding eigenvectors of $k$ largest singular values of $\tilde{P}_T$, $D_T\in \mathbb{R}^{k\times k}$ is an invertible diagonal matrix and $R_T\in \mathbb{R}^{k\times k}$ is a unitary matrix. In the next step, we analyze the relationship between $\tilde{P}_M$ and $\tilde{P}_T$. Considering the $(x_v,x_v')$-th element of $\tilde{P}_M\tilde{P}_M^\top$, we have
    \begin{align*}
    (\tilde{P}_M\tilde{P}_M^\top)_{x_v,x_v'} &= \sum\limits_{x_l} (\tilde{P}_M)_{x_v,x_l} (\tilde{P}_M)_{x_v',x_l}\\
    &=\sum\limits_{x_l} \frac{\gP_M(x_v,x_l)\gP_M(x_v',x_l)}{\gP_L(x_l)\sqrt{\gP_V(x_v)\gP_V(x_v')}}\\
    &=\frac{1}{\sqrt{\gP_V(x_v)\gP_V(x_v')}} \sum\limits_{x_l} \gP_L(x_l)\gP_M(x_v|x_l)\gP_M(x_v'|x_l)
    &\left(\gP_M(x_v,x_l) = \gP_M(x_v|x_l)\gP_L(x_l)\right)\\
    &=\frac{\E_{x_l}\gP_M(x_v|x_l)\gP_M(x_v'|x_l)  }{\sqrt{\gP_V(x_v)\gP_V(x_v')}}\\
    &= (\tilde{P}_T)_{x_v,x_v'}.
    \end{align*}
We know that $\tilde{P}_T = \tilde{P}_M\tilde{P}_M^\top$, so $\tilde{P}_T$ and $\tilde{P}_M$ share the same eigenvectors, i.e., $U^k= U_T^k$. As $D, D_2, R, D_T, R_T$ are invertible matrices and the product of the invertible matrices is still invertible, we obtain
\begin{align*}
    F_V ^\star =  (F_V ^\star)' T,
\end{align*}
where $T = (D_T)^{-1} (R_T)^{-1} D R$ is an invertible matrix. 
With Lemma \ref{lem:linear absorbed}, we obtain
\begin{align*}
   \mathcal{E}(f_V^\star) = \mathcal{E}(f_V'^\star), 
\end{align*}
where $(F_V^\star)_{x_v} = f_V^\star(x_v)^\top, (F_V^\star)'_{x_v}=f_V'^\star(x_v)^\top $. So Theorem \ref{thm:multi-uni} is proved.
\end{proof}

\subsection{Proof of Theorem \ref{thm:eigenvalues of hierarchal graph}}

\begin{proof}
The co-occurrence matrix of the three-layer hierarchical random graph is:
$$ P = \begin{pmatrix}
p_h & \cdots & p_h & p_l & \cdots & p_l & \cdots & p_l & \cdots & p_l \\
\cdots & \cdots & \cdots & \cdots & \cdots & \cdots & \cdots & \cdots & \cdots & \cdots \\
p_h & \cdots & p_h & p_l & \cdots & p_l & \cdots & p_l & \cdots & p_l \\
p_l & \cdots & p_l & p_h & \cdots & p_h & \cdots & p_l & \cdots & p_l \\
\cdots & \cdots & \cdots & \cdots & \cdots & \cdots & \cdots & \cdots & \cdots & \cdots \\
p_l & \cdots & p_l & p_h & \cdots & p_h & \cdots & p_l & \cdots & p_l \\
\cdots & \cdots & \cdots & \cdots & \cdots & \cdots & \cdots & \cdots & \cdots & \cdots \\
p_l & \cdots & \cdots & \cdots & \cdots & p_l & \cdots & p_h & \cdots & p_h \\
\cdots & \cdots & \cdots & \cdots & \cdots & \cdots & \cdots & \cdots & \cdots & \cdots \\
p_l & \cdots & \cdots & \cdots & \cdots & p_l & \cdots & p_h & \cdots & p_h 
\end{pmatrix} .$$

Then we consider the process of computing the eigenvalues:
\begin{align*}
 \vert \sigma E - P\vert &= \begin{vmatrix}
\sigma-p_h & \cdots & -p_h & -p_l & \cdots & -p_l & \cdots & -p_l & \cdots & -p_l \\
\cdots & \cdots & \cdots & \cdots & \cdots & \cdots & \cdots & \cdots & \cdots & \cdots \\
-p_h & \cdots & \sigma-p_h & -p_l & \cdots & -p_l & \cdots & -p_l & \cdots & -p_l \\
-p_l & \cdots & -p_l & \sigma-p_h & \cdots & -p_h & \cdots & -p_l & \cdots & -p_l \\
\cdots & \cdots & \cdots & \cdots & \cdots & \cdots & \cdots & \cdots & \cdots & \cdots \\
-p_l & \cdots & -p_l & -p_h & \cdots & \sigma-p_h & \cdots & -p_l & \cdots & -p_l \\
\cdots & \cdots & \cdots & \cdots & \cdots & \cdots & \cdots & \cdots & \cdots & \cdots \\
-p_l & \cdots & \cdots & \cdots & \cdots & -p_l & \cdots & \sigma-p_h & \cdots & -p_h \\
\cdots & \cdots & \cdots & \cdots & \cdots & \cdots & \cdots & \cdots & \cdots & \cdots \\
-p_l & \cdots & \cdots & \cdots & \cdots & -p_l & \cdots & -p_h & \cdots & \sigma-p_h 
\end{vmatrix}.\\
\end{align*}
We denote that the first layer has $s_l$ branches and the second layer has $s_h$ branches. Add every column to the first column:
\begin{align*}
\begin{vmatrix}
\sigma-s_h*p_h-(s_l-1)*s_h*p_l & \cdots & -p_h & -p_l & \cdots & -p_l & \cdots & -p_l & \cdots & -p_l \\
\cdots & \cdots & \cdots & \cdots & \cdots & \cdots & \cdots & \cdots & \cdots & \cdots \\
\sigma-s_h*p_h-(s_l-1)*s_h*p_l & \cdots & \sigma-p_h & -p_l & \cdots & -p_l & \cdots & -p_l & \cdots & -p_l \\
\sigma-s_h*p_h-(s_l-1)*s_h*p_l & \cdots & -p_l & \sigma-p_h & \cdots & -p_h & \cdots & -p_l & \cdots & -p_l \\
\cdots & \cdots & \cdots & \cdots & \cdots & \cdots & \cdots & \cdots & \cdots & \cdots \\
\sigma-s_h*p_h-(s_l-1)*s_h*p_l & \cdots & -p_l & -p_h & \cdots & \sigma-p_h & \cdots & -p_l & \cdots & -p_l \\
\cdots & \cdots & \cdots & \cdots & \cdots & \cdots & \cdots & \cdots & \cdots & \cdots \\
\sigma-s_h*p_h-(s_l-1)*s_h*p_l & \cdots & \cdots & \cdots & \cdots & -p_l & \cdots & \sigma-p_h & \cdots & -p_h \\
\cdots & \cdots & \cdots & \cdots & \cdots & \cdots & \cdots & \cdots & \cdots & \cdots \\
\sigma-s_h*p_h-(s_l-1)*s_h*p_l & \cdots & \cdots & \cdots & \cdots & -p_l & \cdots & -p_h & \cdots & \sigma-p_h
\end{vmatrix}.
\end{align*}
For the $i$-row, if $i$ is not divisible by $s_h$, then minus the row by $(i|s_h *s_h)$-row, and we obtain
\begin{align*}
\begin{vmatrix}
\sigma-s_h*p_h-(s_l-1)*s_h*p_l & \cdots & -p_h & -p_l & \cdots & -p_l & \cdots & -p_l & \cdots & -p_l \\
\cdots & \cdots & \cdots & \cdots & \cdots & \cdots & \cdots & \cdots & \cdots & \cdots \\
0 & \cdots & \sigma & 0 & \cdots & 0 & \cdots & 0 & \cdots & 0 \\
\sigma-s_h*p_h-(s_l-1)*s_h*p_l & \cdots & -p_l & \sigma-p_h & \cdots & -p_h & \cdots & -p_l & \cdots & -p_l \\
\cdots & \cdots & \cdots & \cdots & \cdots & \cdots & \cdots & \cdots & \cdots & \cdots \\
0 & \cdots & 0 & -\sigma & \cdots & \sigma & \cdots & 0 & \cdots & 0 \\
\cdots & \cdots & \cdots & \cdots & \cdots & \cdots & \cdots & \cdots & \cdots & \cdots \\
\sigma-s_h*p_h-(s_l-1)*s_h*p_l & \cdots & \cdots & \cdots & \cdots & -p_l & \cdots & \sigma-p_h & \cdots & -p_h \\
\cdots & \cdots & \cdots & \cdots & \cdots & \cdots & \cdots & \cdots & \cdots & \cdots \\
0 & \cdots & \cdots & \cdots & \cdots & 0 & \cdots & -\sigma & \cdots & \sigma
\end{vmatrix}.
\end{align*}
For the $j$-column that satisfies $j$ is divisible by $s_h$ and $0<j\leq(s_l-1)*(s_h)$, add $\{j+1,\cdots, j+s_h\}$-columns, and for the $j$-column that satisfies $j$ is divisible by $s_h$ and $0<j<(s_l-1)*(s_h)$, minus $\{j+s_h+1,\cdots, j+2*s_h\}$-columns to the $j$-column, then we have
\begin{align*}
\begin{vmatrix}
\sigma-s_h*p_h-(s_l-1)*s_h*p_l & \cdots & -p_h & 0& \cdots & -p_l & \cdots & -s_h*p_l & \cdots & -p_l \\
\cdots & \cdots & \cdots & \cdots & \cdots & \cdots & \cdots & \cdots & \cdots & \cdots \\
0 & \cdots & \sigma & 0 & \cdots & 0 & \cdots & 0 & \cdots & 0 \\
\sigma-s_h*p_h-(s_l-1)*s_h*p_l& \cdots & -p_l & \sigma-s_h*(p_h-p_l) & \cdots & -p_h & \cdots & -s_h*p_l & \cdots & -p_l \\
\cdots & \cdots & \cdots & \cdots & \cdots & \cdots & \cdots & \cdots & \cdots & \cdots \\
0 & \cdots & 0 & 0 & \cdots & \sigma & \cdots & 0 & \cdots & 0 \\
\cdots & \cdots & \cdots & \cdots & \cdots & \cdots & \cdots & \cdots & \cdots & \cdots \\
\sigma-s_h*p_h-(s_l-1)*s_h*p_l & \cdots & \cdots & \cdots & \cdots & -p_l & \cdots & \sigma-s_h*p_h & \cdots & -p_h \\
\cdots & \cdots & \cdots & \cdots & \cdots & \cdots & \cdots & \cdots & \cdots & \cdots \\
0 & \cdots & \cdots & \cdots & \cdots & 0 & \cdots & 0& \cdots & \sigma
\end{vmatrix}.
\end{align*}
When expanding the determinant, the $i$-row that satisfies $i$ is not divisible by $s_h$ only has one non-zero value $\sigma$ in $i$ column, so the det is equal to
\begin{align*}\sigma^{(s_l-1)*s_h}
\begin{vmatrix}
\sigma-s_h*p_h-(s_l-1)*s_h*p_l & 0 & \cdots & 0 & -s_h*p_l  \\
\sigma-s_h*p_h-(s_l-1)*s_h*p_l & \sigma-s_h*(p_h-p_l) & \cdots & 0 & -s_h*p_l  \\
\sigma-s_h*p_h-(s_l-1)*s_h*p_l & -\sigma+s_h*(p_h-p_l) & \cdots & 0 & -s_h*p_l\\
\cdots & \cdots & \cdots & \cdots & \cdots  \\
\cdots & \cdots & \cdots & -\sigma+s_h*(p_h-p_l) & \sigma-s_h*p_h \\
\end{vmatrix}.\\
\end{align*}
The form of the det is easy to expand and we obtain the results:
\begin{equation}
   \sigma^{(s_l-1)*s_h}*( \sigma-s_h*p_h-(s_l-1)*s_h*p_l)*(\sigma-s_h*(p_h-p_l))^{s_h-1}.
\end{equation}
So the eigenvalues are
    \begin{align*}
        &\sigma_1 = s_h*p_h+(s_l-1)*s_h*p_l = \frac{1}{s_l*s_h},\\
        &\sigma_2 = \cdots = \sigma_{s_l} = s_h*(p_h - p_l),\\
        &\sigma_{s_l+1} = \cdots = \sigma_{s_l*s_h} = 0.
    \end{align*}
where $\frac{1}{s_l*s_h}$ and 0 are constants. As the matrix is a real symmetric matrix, the eigenvalues are equal to the singular values. And the row sum of the matrix is a constant, so we can obtain the results of Theorem \ref{thm:eigenvalues of hierarchal graph}.
\end{proof}

\end{document}